\documentclass[preprint]{article}

\usepackage{neurips_2026}

\usepackage[utf8]{inputenc}
\usepackage[T1]{fontenc}
\usepackage{hyperref}
\usepackage{url}
\usepackage{booktabs}
\usepackage{amsfonts}
\usepackage{amsmath}
\usepackage{amssymb}
\usepackage{amsthm}
\usepackage{mathtools}
\usepackage{microtype}
\usepackage{xcolor}
\usepackage{tikz}
\usepackage{tikz-cd}
\usepackage{natbib}

\newcommand{\Lan}{\mathrm{Lan}}
\newcommand{\Ran}{\mathrm{Ran}}
\newcommand{\Nat}{\mathrm{Nat}}

\newcommand{\colim}{\operatorname*{colim}}
\newcommand{\UDL}{\operatorname{UDL}}
\newcommand{\ctx}{\mathcal{C}}
\newcommand{\obs}{\mathcal{D}}
\newcommand{\probe}{\mathcal{P}}
\newcommand{\dec}{\mathcal{E}}

\newtheorem{theorem}{Theorem}

\newtheorem{corollary}[theorem]{Corollary}
\newtheorem{assumption}[theorem]{Assumption}
\newtheorem{definition}[theorem]{Definition}
\newtheorem{example}[theorem]{Example}
\newtheorem{remark}[theorem]{Remark}

\title{Universal Decision Learners}

\author{%
  Sridhar Mahadevan\thanks{Academic affiliation: Research Professor, University of Massachusetts, Amherst; See webpage at \url{https://people.cs.umass.edu/~mahadeva/Site/About_Me.html}} \\
  Adobe Research\\
  San Jose, CA\\
  \texttt{smahadev@adobe.com} \\
}

\begin{document}

\maketitle

\begin{abstract}
Many theories of decision making turn local information into globally coherent behavior.  We propose a typed categorical formulation in which a \emph{Universal Decision Learner} (UDL) first extends observations to a category of decision probes by a left Kan extension and then assembles those probes over a global context category by a right Kan extension.  The two stages use distinct indexing functors, so their composite is well-defined.  We separate the standard universal properties of Kan extensions from the decision-theoretic contribution: a common interface for specifying local evidence, global consistency tests, algorithm-comparison maps, and semantics-preserving abstractions.  For finite discounted Markov decision processes, we derive the Bellman optimality operator exactly as the restriction of a pointwise right Kan extension to state objects; we also give concrete planning, strategic-equilibrium, online-learning, causal-query, and decentralized-decision interpretations and prove the correct kernel-congruence universal property for semantic quotients.  The framework is a semantic comparison language rather than a new numerical solver.
\end{abstract}

\section{Introduction}

Learning to make optimal decisions has been studied in many contexts. Learning to plan constructs policies from action models and trajectories.  Learning to act with reinforcement learning constructs policies from sample trajectories labeled with rewards, state transitions, and yields approximately optimal value functions.  Causal discovery uses interventions as constraints on structural models.  Learning a Nash equilibrium in game theory attempts to find an equilibrium from a database of past strategic interaction.  Online learning uses the  history of sequential losses to find decisions that minimize long-term regret. Viewed abstractly, all these settings involve Universal Decision Learning (UDL): how to extend partial local information into long-term optimal decisions that are coherent in contexts not directly observed.

Attempting to formalize UDL in the traditional setting of function approximation over sets (e.g., metric spaces or vector spaces) does not yield an overarching theory, because the function approximation problem has no canonical solution. However,  category theory supplies canonical solutions to this problem through Kan extensions \citep{kan,maclane:71,riehl2017category}.  Let $\obs$ be a category of locally observed contexts, let $\ctx$ be a larger category of contexts in which decisions must be made, and let
\[
J:\obs\to\ctx,\qquad F:\obs\to\dec
\]
encode the inclusion of observed contexts and the local decision data.  A decision learner must construct a functor $\ctx\to\dec$ agreeing with, or coherently related to, $F$. 

The paper's main thesis is that UDL can be formalized as Kan extensions, whose universal properties are uniquely guaranteed:
\begin{quote}
\emph{Learning to make decisions is the problem of canonically extending partial decision data to new contexts.}
\end{quote}

In this view, a left Kan extension $\Lan_JF$ is the universal way to push local decision data forward into larger contexts.  It captures rollout, interpolation, aggregation, search, and reward accumulation.  Dually, a right Kan extension $\Ran_JF$ is the universal way to assemble a decision from constraints imposed by its continuations or observations.  It captures backward consistency, Bellman-style fixed points, feasibility constraints, and coinductive semantics.

\paragraph{Contributions.}
This paper contributes:
\begin{itemize}
\item a type-correct UDL construction with distinct observation, probe, and global-context categories;
\item a decomposition of decision learning into empirical rollout/aggregation by left Kan extension and global consistency by right Kan extension;
\item a comparison criterion, inherited from the right-Kan adjunction, for testing whether an algorithm computes the declared decision semantics;
\item a kernel-congruence account of the coarsest quotient preserving a specified decision semantics;
\item an exact discounted-MDP derivation and concrete, explicitly scoped interpretations for other Universal Decision Model applications.
\end{itemize}

The paper is theoretical.  Its goal is to identify a unifying semantic structure, not to introduce a new empirical benchmark or neural architecture.

\section{Related Work}
\label{sec:related-work}

This paper builds on many lines of work that have mostly developed separately over the past century.  Category theory supplies the general language of universal constructions, functors, limits, colimits, and Kan extensions \citep{kan,maclane:71,riehl2017category}.  Our contribution is to make Kan extension the central semantic operation for decision learning: local decision data are extended to global behavior by rollout-like left Kan extensions and consistency-like right Kan extensions.

A detailed categorical foundation for Artificial General Intelligence (AGI) contains a broader study of Universal Decision Learners and related Lean formalizations, as well as an extended treatment of applications to reinforcement learning \citep{mahadevancategoriesforagi}.  

Witsenhausen's intrinsic model \citep{witsenhausen:1975} was an early attempt to place sequential stochastic control \citep{bertsekas:rlbook}, decentralized decision making, and game-like information structures \citep{vonneumann1947,nash,Maschler_Solan_Zamir_2013} in a common mathematical form .  Its key insight is that decision problems are governed not only by dynamics and objectives, but by which information is available to which decision maker at each decision event.  UDL inherits this concern with information structure, but recasts it in the categorical language of Kan extensions: information-local decision laws are functors on observed contexts, and global decision semantics are their universal extensions.

The Universal Decision Model (UDM) framework of \citet{sm:udm} proposed a categorical organization of causal inference \citep{pearl-book}, Markov decision processes \citep{bertsekas:rlbook}, reinforcement learning \citep{DBLP:books/lib/SuttonB98}, and Witsenhausen's decentralized information structures.  UDM supplies decision makers, product decision spaces, information fields, observation maps, and a notion of solvability.  The present paper does not replace those objects.  It adds an extension-and-comparison layer: given a UDM and partial behavioral data, which probe semantics are induced, which global models satisfy them, when do two algorithms compute the same semantics, and which quotients preserve that semantics?  Section~\ref{sec:special-cases} makes this relationship explicit beyond the RL instance.

Reinforcement learning (RL) is essentially approximate dynamic programming, iteratively computing Bellman fixed points through stochastic approximation via temporal-difference learning \citep{samuel1959,DBLP:books/lib/SuttonB98,bertsekas:rlbook,kushner2003stochastic}.  Categorically, the dynamical systems studied in RL, such as MDPs, partially observable MDPs, predictive state representations (PSRs) \citep{littman2001predictive} etc., are all special cases of universal coalgebras \citep{jacobs:book}, which provides the natural categorical language for unfolding dynamics, bisimulation, and final semantics \citep{rutten2000universal,jacobs:book,SOKOLOVA20115095}.  In this paper, coalgebras describe the dynamics, while right Kan extensions describe the value semantics assigned to those dynamics.

Finally, UDL provides a categorical language for representation discovery in decision making.  Classical bisimulation identifies states with indistinguishable rewards and transitions; proto-value functions and successor representations use transition geometry to define useful value bases or predictive features \citep{mahadevan2005proto,dayan1993successor}.  Our Kan-invariance and homotopy-Kan-invariance notions express the same broad goal at the semantic level: abstractions should preserve the universal extension that determines observable decision behavior.

\section{Decision Data as a Partial Functor}
\label{sec:partial}

We begin with the smallest amount of structure needed to state the framework, and progressively refine it through the main paper and Supplementary Materials. 

\begin{definition}[Decision context]
A \emph{decision context category} $\ctx$ is a category whose objects are contexts at which decisions, values, policies, interventions, or predictions may be evaluated.  Morphisms encode admissible transitions, refinements, observations, continuations, or information-preserving maps between contexts.
\end{definition}

\begin{definition}[Local decision model]
Let $J:\obs\to\ctx$ be a functor from an observed or locally accessible subcategory into a context category.  A \emph{local decision model} is a functor
\[
F:\obs\to\dec,
\]
where $\dec$ is a category of decision objects, such as ordered sets of utilities, vector spaces, probability distributions, policy spaces, or semiring-enriched values.
\end{definition}

The inclusion $J$ says which contexts have direct local information.  The functor $F$ says what that information is.  The core problem is to extend $F$ from $\obs$ to $\ctx$ in a way determined by the structure of the problem rather than by arbitrary modeling choices.

\begin{assumption}[Existence of extensions]
\label{ass:complete}
Throughout the main development, we assume that $\dec$ has the limits and colimits needed for the Kan extensions under discussion.  (Co)limits are universal representations that define objects that most closely ``approximate" diagrams, which are functors from an indexing category into the concrete UDL category. Many instances of (co)limits exist in traditional decision making, such as $\min, \max, \sup$, as well as set-theoretic intersections, products, and pullbacks/pushforwards.  We are usually interested in enriched categories, where the set of arrows is usually given additional structure, and we assume enriched limits and colimits exist.
\end{assumption}

\subsection{Left Kan Extension: Rollout and Aggregation}

The left Kan extension of $F$ along $J$ is the universal functor $\Lan_JF:\ctx\to\dec$ receiving a natural transformation from $F$ through $J$.  Pointwise, when the relevant colimits exist,
\[
(\Lan_JF)(c)\cong \colim_{(Jd\to c)\in (J\downarrow c)}F(d).
\]
Thus $\Lan_JF$ computes the decision object at $c$ by aggregating all observed ways of reaching or explaining $c$. Concretely, in a planning problem, the comma category $(J\downarrow c)$ indexes candidate partial trajectories ending at $c$.  In interpolation, it indexes observed samples related to a new query.  In attention-like mechanisms, it indexes available keys or memories relevant to a query, while the colimit is replaced by an enriched or weighted aggregation.

\subsection{Right Kan Extension: Consistency and Constraint Satisfaction}

The right Kan extension of $F$ along $J$ is the universal functor $\Ran_JF:\ctx\to\dec$ mapping to $F$ through $J$.  Pointwise,
\[
(\Ran_JF)(c)\cong \lim_{(c\to Jd)\in (c\downarrow J)}F(d).
\]
Thus $\Ran_JF$ computes the decision object at $c$ as the object compatible with all observed consequences, restrictions, or continuations of $c$. This dual construction is the natural home for Bellman consistency, feasibility constraints, posterior compatibility, equilibrium conditions, and other fixed-point-like requirements.  In each case, a global decision is not merely generated; it must be compatible with a structured family of local tests.

\section{Universal Decision Learners}
\label{sec:udl}

\begin{definition}[Universal Decision Learner]
\label{def:udl}
Let
\[
J:\obs\to\probe,\qquad
R:\probe\to\ctx,\qquad
F:\obs\to\dec .
\]
The functor $J$ maps observations into rollout or test probes, while $R$ places
those probes in the global context category.  When the required extensions
exist, define
\[
H=\Lan_JF:\probe\to\dec,
\qquad
\UDL_{J,R}(F)=\Ran_RH=\Ran_R(\Lan_JF):\ctx\to\dec .
\]
This is the \emph{Universal Decision Learner} determined by $(J,R,F)$.
\end{definition} 

Every domain and codomain in Definition~\ref{def:udl} matches: the left stage
lands in $\dec^\probe$, which is the domain of the right-Kan functor
$\Ran_R:\dec^\probe\to\dec^\ctx$.  If a rollout is instead constructed
directly on $\ctx$ using $K:\obs\to\ctx$, it must first be restricted to the
probe category; the aligned construction is
\[
\Ran_R\!\left(R^\ast\Lan_KF\right).
\]
The untyped expression $\Ran_J(\Lan_JF)$ is not used: $\Lan_JF$ has domain
$\ctx$ when $J:\obs\to\ctx$, whereas $\Ran_J$ expects an input with domain
$\obs$.

The definition should be read operationally as a two-stage abstraction:
\[
\text{local decision data}
\xrightarrow{\;\Lan_J\;}
\text{rolled-out candidates}
\xrightarrow{\;\Ran_R\;}
\text{globally consistent decisions}.
\]
The first stage expands, aggregates, or propagates information; the second filters, glues, or enforces consistency.

\begin{remark}
The order and indexing functors may vary by application.  For example, model-based planning may emphasize $\Lan$ first, while policy evaluation may emphasize a right Kan fixed point.  Definition~\ref{def:udl} isolates the recurring semantic pattern rather than prescribing a single numerical algorithm.
\end{remark}

\subsection{Reward-Enriched Form}

  \begin{table}
        \centering
        \begin{tabular}{|l|c|c|l|}
            \hline
            \textbf{Objective} & $\oplus$ & $\otimes$ & \textbf{Semiring} \\
            \hline
            Standard RL & $+$ & $\times$ & Probability $(+, \times)$ \\
            Shortest Path & $\min$ & $+$ & Tropical $(\min, +)$ \\
            Reachability & $\lor$ & $\land$ & Boolean $(\lor, \land)$ \\
            Robustness & $\max$ & $\min$ & Bottleneck $(\max, \min)$ \\
            \hline
        \end{tabular}
        \caption{The symmetric monoidal category of reward semirings captures many different special cases.}
        \label{semiring}
\end{table}

Many decision problems are naturally enriched over an ordered or semiring-valued category (see Table~\ref{semiring}).   Let $\dec$ be enriched over the max-plus semiring
\[
(\mathbb{R}\cup\{-\infty\},\oplus=\max,\otimes=+).
\]
If $w(d\to c)$ is the reward or negative cost assigned to a transition, then a pointwise left Kan extension takes the familiar dynamic-programming form
\[
(\Lan_JF)(c)=
\max_{Jd\to c}\bigl(F(d)+w(d\to c)\bigr).
\]
The dual consistency operation may be written schematically as
\[
(\Ran_JF)(c)=
\lim_{c\to Jd}F(d),
\]
or, in ordered reward-enriched settings, as the tightest value satisfying the downstream inequalities induced by all continuations $c\to Jd$.

\begin{example}[A one-step planning recurrence]
Let contexts be nodes in a finite directed graph, let $F(g)$ be the terminal reward at a goal $g$, and let path composition add rewards.  The left Kan extension over paths computes the best accumulated reward among all ways to reach a context.  The right Kan consistency condition requires these local path values to agree with the values assigned to continuations.  In max-plus notation this yields the usual recurrence
\[
V(s)=\max_{a}\{r(s,a)+V(T(s,a))\},
\]
when transitions are deterministic.
\end{example}

\section{Universal Property}
\label{sec:universal}

The next two statements are standard universal properties of Kan extensions
\citep{maclane:71,riehl2017category}.  We record them to fix the comparison
maps used by UDL; they are background results, not new category theory.

\begin{theorem}[Canonical rollout]
\label{thm:left-universal}
Let $J:\obs\to\ctx$ and $F:\obs\to\dec$ be as above.  If $\Lan_JF$ exists, then for every global decision model $G:\ctx\to\dec$ there is a natural bijection
\[
\Nat(\Lan_JF,G)\cong \Nat(F,G\circ J).
\]
Thus $\Lan_JF$ is initial among global decision models receiving the local data $F$.
\end{theorem}

\begin{proof}
This is the defining universal property of the left Kan extension \citep{maclane:71}.  In decision-theoretic terms, specifying a comparison from the rolled-out model $\Lan_JF$ to any global model $G$ is equivalent to specifying how the local data $F$ compare with the restriction of $G$ to observed contexts.  An object is {\em initial} in a category if there is a unique arrow {\em from} it to every other object. 
\end{proof}

\begin{theorem}[Canonical consistency]
\label{thm:right-universal}
If $\Ran_JF$ exists, then for every global decision model $G:\ctx\to\dec$ there is a natural bijection
\[
\Nat(G,\Ran_JF)\cong \Nat(G\circ J,F).
\]
Thus $\Ran_JF$ is terminal among global decision models whose restrictions are compatible with the local data $F$.
\end{theorem}

\begin{proof}
This is the defining universal property of the right Kan extension.  It says that every globally defined model satisfying the local tests factors uniquely through the canonical consistent extension.  An object is {\em terminal} in a category if there is a unique arrow {\em into} it from every other object. 
\end{proof}

\begin{corollary}[UDL comparison principle]
\label{cor:udl-comparison}
Let $J:\obs\to\probe$, $R:\probe\to\ctx$, and
$H=\Lan_JF:\probe\to\dec$.  If $\Ran_RH$ exists, then for every
$G:\ctx\to\dec$ there is a natural bijection
\[
\Nat(G,\UDL_{J,R}(F))
\cong
\Nat(R^\ast G,H).
\]
Consequently, a comparison $R^\ast G\to H$ determines a unique comparison
$G\to\UDL_{J,R}(F)$.  When the latter is a natural isomorphism, $G$ computes
the same declared decision semantics as the UDL.
\end{corollary}

\begin{proof}
This is the adjunction $R^\ast\dashv\Ran_R$ applied to
$H=\Lan_JF$.  The substantive content in an application is therefore not the
adjunction itself, but the specification of $\obs,\probe,\ctx,J,R,F$ and the
proof that a concrete decision operator realizes the resulting pointwise Kan
formula.
\end{proof}

The corollary gives a precise sense in which UDL semantics are canonical.  A proposed decision rule may have many computational implementations, but it represents the same semantic solution exactly when it satisfies the same universal comparison property.

\section{Fixed Points and Coinduction}
\label{sec:fixed-points}

Many decision algorithms can be described through iterative fixed-point
computations.  Bellman equations and some equilibrium or recursive structural
equations seek objects stable under a declared operator.  A right Kan extension
does not automatically create such an endofunctor or prove convergence.
Rather, an application must first define
$T:\dec^\ctx\to\dec^\ctx$ from its local continuation data and then prove that
the pointwise right Kan construction realizes $T$.  A globally coherent
decision model is subsequently characterized by
\[
V\cong T(V).
\]
When $T$ is contractive in a complete metric enrichment, iteration converges
to its unique fixed point.  Section~\ref{sec:mdp-ran} verifies each of these
steps for a finite discounted MDP.

\begin{example}[Bellman consistency]
For a discounted Markov decision process, define
\[
(TV)(s)=\max_a\sum_{s'}P(s'|s,a)\bigl(r(s,a,s')+\gamma V(s')\bigr).
\]
The optimal value $V^\ast$ is the unique fixed point of $T$ under standard
assumptions \citep{DBLP:books/lib/SuttonB98,bertsekas:rlbook}.
Section~\ref{sec:mdp-ran} constructs a category in which the backup $T(V)$,
for each fixed $V$, is a pointwise right Kan extension.  The fixed-point
equation is an additional metric statement.
\end{example}

\begin{example}[Structural causal models]
A structural causal model specifies local assignments for variables and
interventions replace selected assignments \citep{pearl-book}.  Let $\obs$
encode locally specified mechanisms and $\ctx$ encode interventional contexts.
A do-intervention changes the local functor $F$ on selected components.  When
the causal factorization has been represented by an appropriate diagram, its
global distribution can be assembled as an application-specific extension.
Identifiability remains a separate theorem: the query must agree across all
causal models compatible with the assumptions and observed law.
\end{example}

\section{Kan Invariance and Abstraction}
\label{sec:abstraction}

A decision model is useful only to the extent that it supports abstraction.  Different representations may encode the same behavior.  UDL expresses behavioral equivalence as invariance under universal extension.

\begin{definition}[Kan-invariant property]
A property of decision models is \emph{Kan-invariant} if it depends only on the induced extension $\widehat F$, where $\widehat F$ is either $\Lan_JF$, $\Ran_JF$, or a UDL composite, and not on the particular syntactic presentation of $F$.
\end{definition}

\begin{definition}[Kan bisimulation]
Let $(J_M:\obs_M\to\ctx_M,F_M)$ and $(J_N:\obs_N\to\ctx_N,F_N)$ be two decision models.  They are \emph{Kan bisimilar} when, after equivalence of their context categories, their induced decision semantics are naturally isomorphic:
\[
\widehat F_M\cong \widehat F_N.
\]
\end{definition}

The direction of extension determines the behavioral notion being preserved.  Left Kan bisimulation identifies systems that generate the same forward aggregate behavior, such as the same reachable outcomes, accumulated rewards, or rollout distributions.  Right Kan bisimulation identifies systems that satisfy the same global consistency constraints, such as the same value functions, equilibria, or fixed-point solutions.

Quotient morphisms give the corresponding notion of abstraction.  If
\[
q:\ctx_M\to \ctx_{M'}
\]
collapses contexts, decisions, or information fields, then $q$ is behavior-preserving exactly when universal extension commutes with quotienting:
\[
\widehat F_M(c)\cong \widehat F_{M'}(q(c))
\qquad \text{for all } c\in\ctx_M .
\]
Thus a quotient is valid not because it is syntactically natural, but because it identifies only contexts whose Kan-extended behavior is already indistinguishable.

In Markov decision processes, ordinary bisimulation identifies states whose
one-step rewards agree and whose transition probabilities into equivalence
classes agree.  Equality of a single optimal value function is weaker: two
non-bisimilar states can have the same optimal value.  A UDL quotient recovers
an MDP-bisimulation quotient only when its semantic functor records the full
one-step reward/transition probes, or equivalently a separating family such as
value semantics for every relevant policy.  A quotient formed from $V^\ast$
alone is merely an optimal-value quotient and is not claimed to be classical
bisimulation.

\begin{definition}[Semantic kernel congruence]
Let $H:\ctx\to\dec$ be a specified induced decision semantics.  A
\emph{semantic kernel congruence} $\sim_H$ is a congruence on the objects and
morphisms of $\ctx$ that identifies exactly the distinctions erased by $H$.
Concretely, parallel morphisms $f,g$ satisfy $f\sim_H g$ when $H(f)=H(g)$,
with the analogous object relation, and the relation is assumed compatible
with domains, codomains, identities, and composition.
\end{definition}

\begin{theorem}[Semantic kernel quotient]
\label{thm:minimal-quotient}
Suppose the quotient by $\sim_H$ exists.  Let
\[
q_H:\ctx\to\ctx/{\sim_H}
\]
be its quotient functor.  Then $H$ factors uniquely as
\[
H=\overline H\circ q_H.
\]
If $q':\ctx\to\mathcal Q$ is any other quotient preserving $H$, so that
$H=H'\circ q'$ for some $H':\mathcal Q\to\dec$, then
\[
\ker(q')\subseteq\ker(H),
\]
and there is a unique functor $u:\mathcal Q\to\ctx/{\sim_H}$ satisfying
\[
q_H=u\circ q'.
\]
Thus the semantic kernel quotient is the coarsest quotient that preserves
exactly the distinctions visible to $H$.
\end{theorem}

\begin{proof}
By definition, $H$ is constant on the object and morphism classes of
$\sim_H$, so the universal property of the quotient gives the unique
$\overline H$ with $H=\overline H\circ q_H$.  If $H=H'\circ q'$, then
$q'(x)=q'(y)$ implies $H(x)=H(y)$, and likewise for parallel morphisms.
Hence $\ker(q')\subseteq\ker(H)=\sim_H$.  Therefore the assignment
$q'(x)\mapsto q_H(x)$ (and the analogous assignment on morphisms) is
well-defined and yields a unique functor $u$ satisfying $q_H=u\circ q'$.
\end{proof}

The factorization direction is important: the semantic quotient factors
through every finer semantics-preserving quotient, not conversely.  The theorem
is relative to the chosen semantics $H$; enriching $H$ with more probes makes
its kernel finer and preserves more behavioral distinctions.

\subsection{Homotopy Kan Invariance}
\label{sec:homotopy-kan}

Kan invariance, as stated above, characterizes behavioral equivalence by equality or natural isomorphism of induced Kan extensions.  In many decision settings this is too strict.  In causal inference, distinct structural models may induce the same observational distribution.  In reinforcement learning, two value semantics may be equivalent up to contraction, approximation, or representation-preserving deformation.  These cases motivate a homotopical refinement of Kan invariance.

Let $F,G:\ctx\to\dec$ be behavioral functors into a category $\dec$ equipped with a notion of homotopy, such as a model category, an $\infty$-category, or a metric-enriched category in which deformation is interpreted by paths or contractions.  A homotopy $F\simeq G$ is a deformation preserving the observable structure relevant to the decision problem.

\begin{definition}[Homotopy Kan equivalence]
Two decision models $M$ and $N$ are \emph{homotopy Kan equivalent} if their induced Kan semantics are equivalent up to homotopy:
\[
\widehat F_M \simeq \widehat F_N,
\]
where $\widehat F_M$ and $\widehat F_N$ denote the relevant left Kan extension, right Kan extension, or UDL composite.
\end{definition}

Homotopy Kan equivalence identifies models whose behaviors can be deformed into one another without changing observable decision consequences.  It generalizes strict Kan bisimulation by replacing exact equality with equivalence up to an admissible deformation:
\begin{center}
\emph{observational equivalence corresponds to homotopy equivalence of Kan extensions.}
\end{center}

This viewpoint is especially natural for causal models.  Two structural causal models may disagree internally while inducing the same observational or interventional semantics over the contexts available to the learner.  In UDL terms, such models should not necessarily be identified as syntactic objects; rather, they belong to the same homotopy class when their Kan-extended observational functors agree up to the chosen equivalence.

The same pattern appears in representation learning and reinforcement learning.  Different state representations, value bases, or approximate models may compute decision semantics that are not literally identical, but remain equivalent after quotienting or deformation by behavior-preserving maps.  Thus the space of UDL objects naturally carries higher-order structure: objects are decision models, morphisms are structure-preserving maps, and homotopies are observational equivalences between such maps.  This suggests that Universal Decision Learners can be organized not only as a category of models, but as a higher category of models, transformations, and equivalences.

\section{Special Cases}
\label{sec:special-cases}

\subsection{From Universal Decision Models to Learners}

UDM specifies the static architecture of a decision problem: decision makers,
their product decision space, information fields, observation maps, objectives,
and a solvability question \citep{sm:udm}.  UDL adds a typed extension problem
on top of that architecture.  An observation functor records the behavioral
data available locally; a probe category records candidate continuations or
consistency tests; and a global context category records where a complete
decision semantics is required.  The relationship is summarized in
Table~\ref{tab:udm-udl}.

\begin{table}[t]
\centering
\small
\begin{tabular}{p{0.29\linewidth}p{0.61\linewidth}}
\toprule
\textbf{UDM component} & \textbf{Role in a UDL instance}\\
\midrule
Decision makers and product decisions &
Objects assigned policies, actions, values, or feasible decision sets.\\
Information fields and observation maps &
The observed category $\obs$ and the map $J:\obs\to\probe$ specifying which
probes have empirical or local support.\\
Dynamics, objectives, and constraints &
The local functor $F$ and the probe semantics $H=\Lan_JF$.\\
Solvability &
Existence of a global model $G:\ctx\to\dec$ whose comparison
$G\to\Ran_RH$ is an isomorphism (or an application-specific approximation).\\
Behavior-preserving abstraction &
The semantic kernel quotient of the declared global semantics $H$.\\
\bottomrule
\end{tabular}
\caption{UDM supplies the decision architecture; UDL supplies an
extension-and-comparison semantics for partial behavioral data.}
\label{tab:udm-udl}
\end{table}

\subsection{Planning}

For deterministic shortest-path planning, let $\obs$ contain an initial state
and observed weighted edges, let $\probe$ be the path category generated by
those edges, and let $J:\obs\to\probe$ insert the local data.  With costs in the
min-plus enrichment, the pointwise left Kan extension at a vertex $v$ is
\[
(\Lan_JF)(v)
=
\min_{\pi:s_0\rightsquigarrow v}
\sum_{e\in\pi}w(e).
\]
Thus exhaustive path enumeration, Bellman--Ford relaxation, and Dijkstra's
algorithm (under nonnegative weights) are different computations of the same
min-plus extension.  This is an exact left-Kan instance; no right stage is
needed unless additional global feasibility constraints are imposed.

\subsection{Reinforcement Learning}

Reinforcement learning emphasizes the right-Kan stage.  Section~\ref{sec:url}
constructs an explicit probe category for a finite discounted MDP and derives
the max--expectation--discount Bellman operator pointwise.  Value iteration and
policy iteration then have the same semantic target but different
computational traces.

\subsection{Decentralized Decision and Information Structures}

In a decentralized UDM, the product decision space
$U=\prod_i U_i$ is paired with information fields describing what each
decision maker observes.  Let $\probe$ contain local policy probes
$q_i$ indexed by those information fields, and let $F(q_i)$ be the admissible
local policies measurable with respect to decision maker $i$'s information.
In a category of feasible policy sets ordered by inclusion, a right-Kan limit
is their compatible intersection after transport to the product decision
space.  A global section is therefore a joint policy whose components obey all
declared information restrictions.  This construction organizes the
compatibility part of UDM solvability; it does not by itself prove that an
optimal decentralized policy exists or is computationally tractable.
Appendix~\ref{app:decentralized} gives the complete finite construction.

\subsection{Game-Theoretic Equilibrium}

Let $\Sigma=\prod_i\Sigma_i$ be the mixed-strategy space of a finite game.
For each player $i$, define the best-response constraint set
\[
B_i=\{\sigma\in\Sigma:
\sigma_i\in\operatorname{BR}_i(\sigma_{-i})\}.
\]
Take one probe object $q_i$ per player, arrows from a global-profile context
to each $q_i$, and let the probe functor return $B_i$.  In the poset of subsets
of $\Sigma$ ordered by inclusion, the pointwise right-Kan limit at the global
context is
\[
\bigcap_i B_i,
\]
the set of Nash equilibria.  This does not provide a new equilibrium algorithm;
it identifies the exact global object against which best-response,
fictitious-play, or other iterative procedures may be compared
\citep{nash,vonneumann1947,Maschler_Solan_Zamir_2013}.
Appendix~\ref{app:games} also derives correlated equilibrium by changing the
probe family.

\subsection{Online Learning}

In online learning \citep{DBLP:conf/stoc/BansalJ0S20}, contexts are histories
\[
h_t=(a_1,\ell_1,\ldots,a_{t-1},\ell_{t-1}),
\]
with morphisms given by history extension.  In the additive enrichment, left
Kan extension along a realized history gives cumulative loss
$L_T(a)=\sum_{t\le T}\ell_t(a)$.  A comparator probe maps each fixed admissible
action to its cumulative loss, and the external-regret statistic is the
comparison
\[
\operatorname{Regret}_T
=
\sum_{t\le T}\ell_t(a_t)
-
\min_{a\in A}\sum_{t\le T}\ell_t(a).
\]
Sublinear regret is therefore a quantitative approximation guarantee relative
to the comparator extension, not an automatic consequence of the existence of
a Kan extension.  Appendix~\ref{app:online} constructs both stages and the
resulting comparison gap explicitly.

\subsection{Causal Intervention}

Causal decision-making \citep{pearl-book} asks how local mechanisms change
under intervention and whether a query is identified from available data.
A UDL instance can organize already specified local mechanisms, adjustment
functionals, or interventional regimes as probes and compare their induced
query values across compatible causal models.  A query is identified only when
all models allowed by the stated causal assumptions and observational law give
the same value.  Kan extension supplies bookkeeping and comparison maps; it
does not replace consistency, positivity, exchangeability, a valid adjustment
criterion, or do-calculus.  This qualification separates categorical
organization from causal identification.  Appendix~\ref{app:causal} gives a
compatible-model construction in which the right-Kan value is a singleton
exactly when the declared scalar query is identified.

\section{Universal Reinforcement Learning}
\label{sec:url}

We now specialize the UDL framework to reinforcement learning.  This section is included in the main body because RL is the cleanest and most familiar fixed-point instance of the general construction: local transition and reward information is extended into globally coherent value semantics. A more extended treatment of URL is given in \citep{mahadevancategoriesforagi}. 

\subsection{MDPs as Coalgebras}

A discounted Markov decision process may be viewed as a coalgebra
\[
\alpha:S\to \mathcal{P}(A\times \mathbb{R}\times S)
\]
in the deterministic case, or
\[
\alpha:S\to (A\to \mathcal{D}(\mathbb{R}\times S))
\]
in the stochastic case, where $\mathcal{D}$ is the probability distribution functor.  The coalgebra unfolds each state into its available actions, rewards, and successor-state distributions.  This is the dynamical substrate over which UDL defines value semantics.

\subsection{The Discounted Bellman Operator as a Pointwise Right Kan Extension}
\label{sec:mdp-ran}

We give an exact construction, rather than relying on an analogy with
``consistency.''  Let
\[
M=(S,\{A_s\}_{s\in S},P,r,\gamma)
\]
be a finite MDP, where $0\leq\gamma<1$ and
$r(s,a)=\sum_{s'}P(s'|s,a)r(s,a,s')$ is bounded.  Let $\probe$ be the
discrete category with one object $q_{s,a}$ for every admissible state--action
pair.  Let $\ctx$ have the objects of $\probe$ together with one object for
each state $s$, and add a morphism
\[
s\longrightarrow q_{s,a}
\]
for every $a\in A_s$.  Write $R:\probe\hookrightarrow\ctx$ for the inclusion.
Let $\mathcal S$ denote the discrete category on the state set and
$I:\mathcal S\hookrightarrow\ctx$ the inclusion of state objects.

Take $\dec$ to be the thin category of extended real numbers in reverse
numerical order.  A finite categorical limit in $\dec$ is therefore a
numerical maximum.  For each candidate value function
$V:S\to\mathbb{R}$, define the probe functor
\[
F_V:\probe\to\dec,\qquad
F_V(q_{s,a})
=
r(s,a)+\gamma\sum_{s'\in S}P(s'|s,a)V(s').
\]
The comma category $(s\downarrow R)$ is discrete and has exactly one object
for each action in $A_s$.  Hence the pointwise right-Kan formula gives
\begin{align}
(\Ran_R F_V)(s)
&=
\lim_{(s\to Rq)\in(s\downarrow R)}F_V(q) \\
&=
\max_{a\in A_s}
\left\{
r(s,a)+\gamma\sum_{s'\in S}P(s'|s,a)V(s')
\right\}\\
&=T(V)(s),
\end{align}
where $T$ is the Bellman optimality operator.  Thus the \emph{backup}
$V\mapsto T(V)$ is exactly the restriction to state objects of a pointwise
right Kan extension in this declared category.  Bellman \emph{optimality} is
the separate fixed-point equation
\[
V^\ast=I^\ast\Ran_R F_{V^\ast}=T(V^\ast).
\]
Because $T$ is a $\gamma$-contraction in the sup norm, Banach's fixed-point
theorem gives a unique $V^\ast$.

This is also a literal instance of the type-correct UDL definition.  For a
known finite MDP, take $\obs=\probe$, $J=\mathrm{id}_{\probe}$, and local data $F_V$.
Then
\[
I^\ast\UDL_{J,R}(F_V)
=
I^\ast\Ran_R(\Lan_{\mathrm{id}_{\probe}}F_V)
\cong
I^\ast\Ran_R F_V
=T(V).
\]
For a learned model, $\obs$ can instead contain sampled reward and transition
records and $J:\obs\to\probe$ can group them by state--action probe.  A chosen
statistical or enriched left-Kan stage produces an estimate
$\widehat F_V=\Lan_JF_V^{\mathrm{obs}}$; the right stage then gives
$\Ran_R\widehat F_V$.  The categorical universal property alone does not
choose an estimator or supply concentration bounds: those require the
probabilistic assumptions of the learning problem.

\subsection{Solution Methods as Approximate Kan Computation}

Exact dynamic-programming algorithms and stochastic approximation methods
must be distinguished.  In the finite, known-model setting, value iteration
and policy iteration compute the unique fixed point just derived.  Sampled
methods target the same Bellman equation only under their usual coverage,
stepsize, tabular, or function-approximation assumptions.
\[
\begin{array}{lll}
\toprule
\text{Method} & \text{Computation} & \text{UDL interpretation}\\
\midrule
\text{Value iteration} & V_{k+1}=TV_k
  & \text{iteration of the exact Kan-derived backup}\\
\text{Policy iteration} & \text{evaluation/improvement}
  & \text{another exact route to the same optimum}\\
\text{Monte Carlo} & \text{return averaging}
  & \text{empirical path-value estimation}\\
\text{TD learning} & \text{bootstrapped updates}
  & \text{stochastic Bellman-residual control}\\
\text{Q-learning} & \text{off-policy sampled backups}
  & \text{tabular convergence under standard conditions}\\
\bottomrule
\end{array}
\]

\begin{theorem}[Agreement of two exact algorithms]
\label{prop:algorithm-agreement}
For the finite discounted MDP of Section~\ref{sec:mdp-ran}, exact value
iteration and exact policy iteration compute the same state-value semantics
induced by the UDL:
\[
V^\ast
=
\operatorname{fix}(T)
=
\operatorname{fix}\!\left(V\mapsto I^\ast\Ran_RF_V\right).
\]
\end{theorem}

\begin{proof}
For value iteration, the equality $I^\ast\Ran_RF_V=T(V)$ proved above and the
$\gamma$-contraction theorem imply
$T^k(V_0)\to V^\ast$ for every bounded $V_0$.  Exact policy iteration
evaluates each stationary policy by the unique fixed point of its discounted
policy operator and applies greedy improvement.  The policy-improvement
theorem gives monotone improvement and, because a finite MDP has finitely many
deterministic stationary policies, termination at an optimal policy (using
conservative tie breaking).  Its value satisfies $T(V)=V$, and uniqueness of
the contraction fixed point makes it $V^\ast$.  The algorithms have different
computational traces but the same Kan-derived backup and fixed-point semantics.
\end{proof}

A general comparison criterion is now explicit: an algorithm representing
$\widetilde V$ implements the declared state-value semantics exactly when its
comparison with $I^\ast\Ran_RF_{\widetilde V}$ is an isomorphism and it satisfies
the fixed-point condition.  In a metric enrichment, the corresponding
residual can converge to zero.  This formulation separates exact dynamic
programming, stochastic approximation, and restricted function approximation
without claiming identical finite-sample behavior.  Appendix
\ref{app:function-approx-url} discusses the last case.

\subsection{Bisimulation and Representation}

Ordinary MDP bisimulation identifies states whose immediate rewards agree and
whose transition probabilities agree after aggregation over equivalence
classes.  Equality of one optimal value function is weaker: two states may
have the same $V^\ast$ while having different rewards or transition laws.
Consequently, the semantic-kernel theorem recovers an MDP-bisimulation
quotient only when the declared semantic functor contains a separating family
of probes recording the relevant reward and class-aggregated transition data
(or an equivalently rich family of policy/value tests).  If the functor records
only $V^\ast$, its kernel is merely the coarsest quotient preserving that
particular value semantics.  This distinction prevents value equivalence from
being mislabeled as probabilistic bisimulation.

\subsection{Beyond MDPs}

The coalgebraic formulation is not limited to classical MDPs.  Partially
observed systems, predictive state representations, non-Markovian histories,
and asynchronous distributed processes can be modeled by changing the
coalgebra and the declared probe and context categories
\citep{jacobs:book,rutten2000universal,SOKOLOVA20115095}.  Each case still
requires a separate existence and pointwise calculation; the finite-MDP
derivation above does not establish those results automatically.  The
reusable specification pattern is
\[
\begin{array}{c}
\text{coalgebraic dynamics}
\;+\;
\text{local reward/observation data}\\
\Longrightarrow
\text{Kan-extended value semantics}.
\end{array}
\]

\section{Discussion and Limitations}

The UDL perspective separates semantics from algorithms.  Once the categories,
functors, enrichment, and comparison maps have been declared, a Kan extension
identifies a canonical semantic object.  Dynamic programming, stochastic
approximation, equilibrium iteration, causal adjustment, or regret
minimization may then compute or approximate the corresponding
application-specific object.  UDL in its current formulation does not propose
a new numerical algorithm, and the existence of a Kan extension does not by
itself establish statistical consistency, convergence rates, causal
identification, or computational tractability.

This limitation suggests three future directions.  First, new decision algorithms can be compared by asking whether they approximate the same Kan extension.  Second, representation learning can be studied through Kan-invariant quotients: a representation is behaviorally adequate when it preserves the universal extension.  Third, richer settings such as partially observed, distributed, hierarchical, or non-Markovian decision processes can be handled by changing the context category rather than abandoning the semantic framework. 

\section{Conclusion}

Universal Decision Learners recast a family of decision specifications as
typed canonical extensions.  A local decision functor records what is directly
known; a left Kan extension along $J:\obs\to\probe$ populates decision probes;
and a right Kan extension along $R:\probe\to\ctx$ assembles those probes over
global contexts.  The distinct indexes make the staged construction
composable.  In the finite discounted-MDP instance, the pointwise right Kan
extension is exactly the Bellman backup, while optimality is its unique
contraction fixed point.  The corrected semantic-kernel theorem states
precisely what a declared semantics-preserving quotient preserves.  The result
is a comparison language for planning, reinforcement learning, causal
queries, online learning, decentralized decisions, and strategic equilibrium,
with application-specific assumptions kept explicit.

\newpage 

\appendix

\section{Changes in This Version}
\label{app:changes}

This version makes the following substantive corrections and extensions.
\begin{enumerate}
\item \textbf{Type-correct staged definition.}
The Universal Decision Learner is now defined using distinct rollout and
consistency indexes,
\[
J:\obs\to\probe,\qquad R:\probe\to\ctx,\qquad
\UDL_{J,R}(F)=\Ran_R(\Lan_JF).
\]
The earlier untyped composite that applied a right Kan extension to a functor
with the wrong domain has been removed.  When rollout is formed directly on
\(\ctx\), the required restriction \(R^\ast\) is written explicitly.
\item \textbf{Corrected quotient theorem.}
The abstraction result is now a semantic-kernel quotient theorem.  It states
the correct factorization direction and distinguishes quotients preserving a
declared semantic functor from stronger notions such as MDP bisimulation.
\item \textbf{Concrete discounted-MDP derivation.}
The finite discounted Bellman backup is derived as the restriction of a
pointwise right Kan extension on an explicitly constructed category.
Bellman optimality is separately identified as the unique fixed point of the
\(\gamma\)-contraction; the paper no longer conflates backup with optimality.
\item \textbf{Broader decision examples and limitations.}
The planning, strategic-equilibrium, online-learning, causal-query, and
decentralized-decision examples are expanded in the supplement.  These are
presented as typed application schemas with application-specific assumptions,
not as claims that the universal property alone supplies numerical algorithms
or convergence guarantees.
\end{enumerate}

\section{Supplementary Materials: Coalgebras and Metric Coinduction}
\label{app:coinduction}

The notion of \emph{coinduction} may be unfamiliar to some readers, but it plays a central role in the categorical view of reinforcement learning used in this paper.  Informally, induction constructs objects from smaller pieces, whereas coinduction characterizes objects by the observable behaviors they sustain.  Classical optimization and stochastic approximation are often presented inductively: start with an initial estimate and refine it step by step.  Reinforcement learning also has a coinductive semantics: value functions are characterized by consistency with future behavior, and hence by fixed points of behavioral operators.

\paragraph{Induction and coinduction.}
In category theory, induction is associated with initial algebras and well-founded constructions.  Coinduction is associated with final coalgebras and potentially non-well-founded structures.  While induction builds finite objects by recursion, coinduction specifies possibly infinite or self-referential objects through their unfolding behavior.
\begin{quote}
\emph{Induction constructs an object from its generators; coinduction characterizes an object by the observations it supports.}
\end{quote}
This makes coinduction especially well suited to dynamical systems, infinite processes, and sequential decision problems.

\paragraph{From learning in the limit to coinductive inference.}
Classical work in machine learning studied identification or learnability in the limit, later extended to the statistical setting by Valiant \citep{DBLP:journals/cacm/Valiant84}.  This tradition is close to inductive inference, where one seeks to recover a target object by finite approximation.  Coinductive inference generalizes the viewpoint to settings where the target is better understood as a process than as a finite structure.  The learner seeks not merely to enumerate a set, but to identify a behaviorally equivalent system.  Universal coalgebras provide the corresponding mathematical language.

\subsection{Coalgebras as Behavioral Models}

A coalgebra for an endofunctor
\[
B:\mathcal{C}\to\mathcal{C}
\]
is a map
\[
\alpha:X\to B(X),
\]
where $X$ is the carrier object and $B$ specifies the type of observable dynamics.  Coalgebras represent systems by how they evolve, rather than by how they are constructed.  They cover deterministic and nondeterministic automata, probabilistic transition systems, game-theoretic and causal processes, and reinforcement learning environments.

\paragraph{Example: labeled transition systems.}
A labeled transition system consists of states $S$, labels $A$, and a relation
\[
\to_S\;\subseteq\;S\times A\times S.
\]
Define the endofunctor
\[
\mathcal{B}(X)=\mathcal{P}(A\times X).
\]
Then the transition system can be represented as a coalgebra
\[
\alpha_S:S\to\mathcal{B}(S),
\qquad
s\mapsto \{(a,s')\mid s\xrightarrow{a}s'\}.
\]
The coalgebra records the observable one-step behavior of each state.

\paragraph{Final coalgebras and behavioral equivalence.}
A final coalgebra is a universal object in the category of coalgebras: every coalgebra admits a unique morphism into it.  This morphism identifies the behavior of a state by its canonical unfolding.  Two states are behaviorally equivalent precisely when they map to the same point of the final coalgebra, or more generally when they are related by bisimulation.  In the language of the main paper, this is a close cousin of Kan invariance: behavioral equivalence arises when different local descriptions induce the same global semantic extension.

\subsection{Coinduction, Kan Extensions, and RL}

The relevance of coinduction to reinforcement learning can now be stated succinctly.  In RL, the value of a state is not defined directly by finite construction, but by a consistency condition relating it to future states.  This is the essence of the Bellman equation.  From the UDL perspective:
\begin{itemize}
\item local reward and transition structure define a functor;
\item global value semantics arise as a right Kan extension;
\item coinduction characterizes the fixed point of this extension process.
\end{itemize}
\begin{quote}
\emph{Coinduction is the semantic principle underlying right Kan extensions in sequential decision systems.}
\end{quote}
Thus, coinduction is not merely an abstract alternative to induction: it is the natural language for value functions, policies, and long-horizon behavior.

\subsection{Dynamic Programming and Metric Coinduction}
\label{app:metric-coinduction}

In reinforcement learning and dynamic programming, convergence is often analyzed in metric spaces.  The classical argument shows that the Bellman operator is a contraction and therefore admits a unique fixed point.  Metric coinduction \citep{kozen} shows that this familiar argument is naturally coinductive.

Let $(V,d)$ be a metric space, and let
\[
F:V\to V
\]
be contractive, meaning that for some $0\leq c<1$,
\[
d(F(u),F(v))\leq c\,d(u,v)
\qquad \forall u,v\in V.
\]
More generally, $F$ is \emph{eventually contractive} if $F^n$ is contractive for some $n\geq1$.  The standard fixed-point iteration
\[
u,\;F(u),\;F^2(u),\ldots
\]
then converges to a unique fixed point $u^\ast$.

\begin{definition}[Coinduction rule]
Let $\phi$ be a closed nonempty subset of a metric space $V$, and let $F:V\to V$ be an eventually contractive map that preserves $\phi$.  Then the unique fixed point $u^\ast$ of $F$ lies in $\phi$.
\end{definition}

This rule is coinductive because it proves a property of the fixed point not by constructing the point explicitly, but by showing that the property is preserved by the unfolding dynamics.  The fixed point belongs to every closed invariant set preserved by the operator.  This is precisely the form of reasoning used in RL: the Bellman operator preserves a class of admissible value functions, its unique fixed point is characterized coinductively, and convergence follows from the metric structure.

Coalgebraically, one organizes contractive systems into a category of coalgebras in which the fixed point corresponds to a final object.  The coinduction rule then states that any property preserved by the coalgebra structure must hold of the final coalgebra.  Metric coinduction is therefore the analytic counterpart to categorical finality.

In Universal Reinforcement Learning, environments are modeled as coalgebras, value functions are characterized by right Kan extensions, Bellman operators provide the metric realization of these extensions, and stochastic approximation algorithms compute them iteratively.
\begin{center}
\emph{Metric coinduction provides the convergence theory for computing coalgebraic right Kan semantics.}
\end{center}

\section{Background: Function Approximation in URL}
\label{app:function-approx-url}

The main paper describes value functions as right Kan extensions of local reward and transition data.  Classical dynamic programming assumes that this extension can be represented exactly.  Function approximation changes the problem: the learner searches inside a restricted hypothesis class, such as linear features, graph-based representations, or neural networks.  From the UDL viewpoint, function approximation is therefore \emph{approximate Kan computation}: rather than compute $\Ran_JF$ exactly, one chooses a parameterized functor
\[
V_\theta:\ctx\to\dec_\theta
\]
and trains it to approximate the universal object
\[
V^\ast \simeq \Ran_JF.
\]

\subsection{Projected Right Kan Semantics}

Let $\mathcal{H}\subseteq \dec^\ctx$ be a hypothesis class of representable global decision models.  Exact URL seeks a globally consistent value semantics $\widehat F=\Ran_JF$.  Approximate URL seeks
\[
V_\theta \in \mathcal{H}
\]
such that $V_\theta$ is close to $\widehat F$ under a task-relevant discrepancy.  Equivalently, if $\Pi_{\mathcal{H}}$ denotes projection or approximation into the hypothesis class, then function approximation targets a projected right Kan equation
\[
V_\theta \approx \Pi_{\mathcal{H}}\,\Ran_JF.
\]
In Bellman form this becomes the familiar projected fixed-point problem
\[
V_\theta \approx \Pi_{\mathcal{H}}\,T(V_\theta),
\]
where $T$ is the Bellman operator.  The categorical content is that the fixed point remains a right Kan consistency condition, but consistency is now enforced only after projection into the representable class.

\subsection{Representations as Approximate Coalgebra Morphisms}

Function approximation also changes the state space.  Let an environment be represented by a coalgebra
\[
\gamma:X\to T(X),
\]
where $T$ packages actions, observations, rewards, and successor structure.  A learned representation is a map
\[
F_\theta:X\to Z,
\]
together with a latent dynamics model
\[
\widehat\gamma_\phi:Z\to T(Z).
\]
Ideally, $F_\theta$ is a coalgebra morphism, meaning that the following square commutes:
\[
\widehat\gamma_\phi\circ F_\theta
=
T(F_\theta)\circ \gamma.
\]
In deep URL this equality is relaxed:
\[
\widehat\gamma_\phi\circ F_\theta
\approx
T(F_\theta)\circ \gamma.
\]
Thus representation learning is the search for an approximate coalgebra morphism that preserves the right Kan value semantics relevant to the decision problem.

\subsection{Loss Decomposition}

This perspective separates three roles in a function-approximation objective:
\[
\mathcal{L}
=
\mathcal{L}_{\mathrm{Kan}}
+
\lambda_{\mathrm{morph}}\mathcal{L}_{\mathrm{morph}}
+
\lambda_{\mathrm{struct}}\mathcal{L}_{\mathrm{struct}}.
\]
Here $\mathcal{L}_{\mathrm{Kan}}$ is a Bellman or TD loss enforcing approximate right Kan consistency; $\mathcal{L}_{\mathrm{morph}}$ penalizes failure of the coalgebra square to commute; and $\mathcal{L}_{\mathrm{struct}}$ encodes additional relational, graph, or diagrammatic constraints.  The objective is not merely to fit rewards, but to learn a representation in which local transition data extend coherently to global value semantics.

For a sample transition $(x,a,r,x')$, a typical TD component has the form
\[
\mathcal{L}_{\mathrm{Kan}}
=
\bigl(Q_\psi(F_\theta(x),a)
-
(r+\gamma \max_{a'}Q_\psi(F_\theta(x'),a'))\bigr)^2.
\]
A morphism or dynamics component may take the form
\[
\mathcal{L}_{\mathrm{morph}}
=
\left\|
\widehat\gamma_\phi(F_\theta(x),a)
-
T(F_\theta)(\gamma(x,a))
\right\|^2,
\]
with the precise expression depending on whether the transition model is deterministic, stochastic, partially observed, or distributional.

\subsection{Diagrammatic Regularization}

Unconstrained neural approximators may satisfy a Bellman objective while violating structural coherence.  URL therefore suggests adding penalties derived from declared diagrams in the context category.  If $E$ is a declared transition or neighborhood relation over states, and
\[
H=[h_1,\ldots,h_{|X|}]^\top,\qquad h_i=F_\theta(x_i),
\]
then a graph-level diagrammatic penalty is
\[
\mathcal{L}_{\mathrm{struct}}
=
\frac{1}{|E|}
\sum_{(i,j)\in E}
\|h_i-h_j\|_2^2.
\]
Equivalently, with graph Laplacian $L=D-A$,
\[
\mathcal{L}_{\mathrm{struct}}
=
\frac{1}{|E|}\operatorname{tr}(H^\top L H).
\]
Higher-order diagrams, such as commuting action paths or typed transition motifs, can similarly be compiled into residuals that penalize noncommuting squares or triangles.  In categorical terms, diagrammatic backpropagation enforces approximate naturality and approximate preservation of the diagrams along which Kan extensions are computed.

\subsection{Geometric Transformers and Relational Bias}

Geometric or relational architectures can be read as choosing a hypothesis class $\mathcal{H}$ that is already adapted to the context category.  Message passing over a transition graph, attention over typed relations, or equivariant processing over structured observations gives the approximator access to the morphisms that organize the decision problem.  The role of such architectures is to bias the learner toward functions for which the projected right Kan equation is easier to satisfy.

This places deep URL in a lineage with proto-value functions and successor representations.  Proto-value functions use graph Laplacian eigenvectors to construct smooth bases over an MDP transition graph \citep{mahadevan2005proto}.  Successor representations encode discounted future occupancy structure \citep{dayan1993successor}.  URL generalizes these ideas by treating transition-aware representation learning as approximate preservation of Kan-extended value semantics:
\[
\text{representation geometry}
\quad\Longleftrightarrow\quad
\text{structure preserved by approximate Kan extension}.
\]

\subsection{Interpretation}

Function approximation in URL can be summarized as follows:
\begin{itemize}
\item exact value semantics are right Kan extensions;
\item approximate value semantics are projected right Kan extensions;
\item learned representations are approximate coalgebra morphisms;
\item Bellman losses enforce semantic consistency;
\item structural losses enforce diagrammatic coherence.
\end{itemize}
This gives a categorical reading of deep RL with structured representations.  A neural learner is not simply fitting a scalar objective; it is approximating a universal extension while preserving as much of the underlying coalgebraic and diagrammatic structure as its hypothesis class allows.

\section{Supplementary Material: Worked Non-RL Specializations}
\label{app:non-rl}

This appendix makes the constructions in Sections~8.4--8.7 explicit.  These
finite examples are deliberately modest.  They show how a declared UDL
recovers a familiar semantic object; they do not turn the universal property
into an existence, identification, convergence-rate, or tractability theorem.

\subsection{Decentralized Decisions: Information-Compatible Policies}
\label{app:decentralized}

Consider a finite decentralized team with world set $\Omega$, agents
$i=1,\ldots,n$, finite action sets $U_i$, and observation maps
$O_i:\Omega\to Y_i$.  Let
\[
\Pi_i=\{\pi_i:\Omega\to\Delta(U_i)\},
\qquad
\Pi=\prod_{i=1}^n\Pi_i
\]
be the ambient behavioral-policy spaces.  The information available to agent
$i$ imposes the constraint
\[
C_i
=
\left\{
\pi\in\Pi:
O_i(\omega)=O_i(\omega')
\Longrightarrow
\pi_i(\omega)=\pi_i(\omega')
\right\}.
\]
Thus $C_i$ consists of joint policy profiles whose $i$th component is
measurable with respect to the partition induced by $O_i$.

Let $\probe$ be the discrete category with objects $q_1,\ldots,q_n$.  Let
$\ctx$ contain those objects and a global team context $g$, with one arrow
$g\to q_i$ per agent, and let $R:\probe\hookrightarrow\ctx$ be the inclusion.
Take
\[
\dec=(\mathcal P(\Pi),\subseteq),
\qquad
H(q_i)=C_i.
\]
Since limits in this poset are intersections and
$(g\downarrow R)$ is the discrete set of agents, the pointwise formula is
\[
(\Ran_RH)(g)
=
\lim_{(g\to Rq_i)\in(g\downarrow R)}H(q_i)
=
\bigcap_{i=1}^n C_i.
\]
The right-Kan value is therefore exactly the set of joint policies satisfying
all decentralized information restrictions.

For a fully specified UDM, this is a UDL instance with
$\obs=\probe$, $J=\mathrm{id}_{\probe}$, and $F=H$:
\[
\UDL_{J,R}(F)
\cong
\Ran_RH.
\]
Additional feasibility constraints can be represented by more probe objects
and intersected in the same limit.  Given a team cost
$c:\Pi\to\mathbb R$, the optimization problem is then
\[
\min_{\pi\in(\Ran_RH)(g)}c(\pi).
\]
The Kan extension constructs the information-compatible feasible set.  It
does not assert that this set is nonempty, that a minimizer exists in an
infinite model, or that decentralized optimization is computationally easy.
Those are separate UDM solvability and optimization questions
\citep{witsenhausen:1975,sm:udm}.

\subsection{Finite Games: Nash and Correlated-Equilibrium Constraints}
\label{app:games}

Let $\Sigma=\prod_i\Delta(A_i)$ be the mixed-strategy space of a finite
normal-form game.  For player $i$, define the best-response constraint set
\[
B_i
=
\left\{
\sigma\in\Sigma:
u_i(\sigma_i,\sigma_{-i})
\geq
u_i(\tau_i,\sigma_{-i})
\ \text{for every }\tau_i\in\Delta(A_i)
\right\}.
\]
Use the same probe shape as above: $\probe$ is discrete on player probes
$q_i$, $\ctx$ adds a global-profile object $g$ with arrows $g\to q_i$, and
$R$ includes the probes.  In
$\dec=(\mathcal P(\Sigma),\subseteq)$, set $H(q_i)=B_i$.  Then
\[
(\Ran_RH)(g)
=
\bigcap_i B_i
=
\operatorname{NE}(G).
\]
Thus Nash equilibrium is exactly the pointwise right-Kan assembly of the
players' simultaneous best-response constraints.  Nash's theorem supplies
nonemptiness for finite games; nonemptiness does not follow merely from the
existence of the limit \citep{nash}.

The construction also separates semantics from algorithms.  A profile
$\widehat\sigma$ returned by fictitious play, regret matching, or another
procedure computes the exact declared semantics only if
$\widehat\sigma\in(\Ran_RH)(g)$.  A numerical relaxation is the exploitability
\[
\epsilon(\widehat\sigma)
=
\max_i\sup_{\tau_i\in\Delta(A_i)}
\left[
u_i(\tau_i,\widehat\sigma_{-i})
-u_i(\widehat\sigma_i,\widehat\sigma_{-i})
\right].
\]
Exact membership is equivalent to $\epsilon(\widehat\sigma)=0$; convergence
of a particular algorithm requires its own analysis.

For correlated equilibrium, take the ambient set to be
$\Delta(\prod_i A_i)$ and one probe for every deviation
$(i,a_i,a_i')$.  The probe returns the half-space of distributions $\mu$
satisfying
\[
\sum_{a_{-i}}\mu(a_i,a_{-i})
\bigl[
u_i(a_i,a_{-i})-u_i(a_i',a_{-i})
\bigr]\geq0.
\]
The pointwise right-Kan limit is the intersection of all deviation
half-spaces, hence exactly the correlated-equilibrium polytope.  Changing the
probe family changes the equilibrium concept transparently.

\subsection{Online Learning: Rollout, Comparator, and Regret}
\label{app:online}

Fix a finite action set $A$ and a loss sequence
$\ell_t:A\to[0,\infty]$ for $t=1,\ldots,T$.  For each fixed action $a$, form a
chain of history objects
\[
h_{0,a}\longrightarrow h_{1,a}\longrightarrow\cdots
\longrightarrow h_{T,a},
\]
and assign edge $h_{t-1,a}\to h_{t,a}$ the cost $\ell_t(a)$.  In the
min-plus cost enrichment, composition adds edge costs.  Let $\obs$ contain the
initial objects $h_{0,a}$, let $\probe$ be the disjoint union of these weighted
chains, and let $J:\obs\hookrightarrow\probe$ be the inclusion.  With
$F(h_{0,a})=0$, the unique path to $h_{T,a}$ makes the pointwise enriched
left-Kan calculation
\[
H(h_{T,a})
=
(\Lan_JF)(h_{T,a})
=
\sum_{t=1}^T\ell_t(a)
=L_T(a).
\]
This is the rollout stage: it produces the terminal cumulative-loss probe for
each fixed comparator action.

Let $\ctx$ contain the entire rollout category $\probe$ and one additional
global comparator context $g$, with arrows $g\to q_a$ to the terminal objects
$q_a=h_{T,a}$.  Write $R:\probe\hookrightarrow\ctx$ for the inclusion.  The
only objects of $(g\downarrow R)$ are the terminal probes.  In the thin cost
category
$([0,\infty],\leq)$, finite limits are numerical minima, so
\[
(\Ran_RH)(g)
=
\min_{a\in A}L_T(a)
=
\min_{a\in A}\sum_{t=1}^T\ell_t(a).
\]
The two stages therefore recover the best fixed-action comparator.  If an
online algorithm selects $a_t$ and incurs
$A_T=\sum_{t=1}^T\ell_t(a_t)$, its external regret is the comparison gap
\[
\operatorname{Regret}_T
=
A_T-(\Ran_RH)(g).
\]
A no-regret guarantee is the quantitative statement
$\operatorname{Regret}_T=o(T)$.  It is not implied by either Kan universal
property; it must be proved from the algorithm and loss assumptions.  The UDL
contribution is to keep the comparator semantics fixed while allowing
different online algorithms to be compared against it.

\subsection{Causal Queries: Identifiability as Probe Agreement}
\label{app:causal}

Let $o$ denote an observational distribution together with a declared causal
model class and assumptions, and suppose the compatible model set
$\mathcal M(o)$ is nonempty.  Fix an intervention $x$ and a scalar causal
query
\[
Q_x:\mathcal M(o)\to\mathcal Y,
\qquad
M\longmapsto Q_x(M),
\]
such as $Q_x(M)=\mathbb E_M[Y\mid\operatorname{do}(X=x)]$.
Let $\probe$ be the discrete category on the compatible models $M$, and let
$\ctx$ add one observational context $o$ with an arrow $o\to M$ for every
$M\in\mathcal M(o)$.  Write $R:\probe\hookrightarrow\ctx$ for inclusion.
Take
\[
\dec=(\mathcal P(\mathcal Y),\subseteq),
\qquad
H_x(M)=\{Q_x(M)\}.
\]
The comma category $(o\downarrow R)$ is precisely the compatible-model
family, hence
\[
(\Ran_RH_x)(o)
=
\bigcap_{M\in\mathcal M(o)}\{Q_x(M)\}.
\]
This limit is a singleton $\{q\}$ exactly when all compatible models give the
same query value $q$.  If two compatible models disagree, the intersection is
empty.  Thus right-Kan probe agreement gives an exact yes/no representation
of identifiability relative to the declared model class.

For comparison, the colimit of the same discrete compatible-model diagram in
the powerset category is
\[
\bigcup_{M\in\mathcal M(o)}\{Q_x(M)\},
\]
the set of all query values compatible with the assumptions.  Equivalently,
this union can be realized as a pointwise left Kan extension in the covariant
context obtained by adjoining arrows $M\to o$.  It is the identified set; the
query is point-identified exactly when this union is a singleton.  The right
limit tests agreement, while the left colimit records the range of ambiguity.

As a concrete example, if a measured covariate $Z$ satisfies a valid
back-door criterion for the effect of $X$ on $Y$, the causal assumptions imply
that every compatible model has
\[
Q_x(M)
=
\sum_z
\mathbb E[Y\mid X=x,Z=z]\,P(Z=z).
\]
The right-Kan intersection is therefore the singleton containing this
adjustment functional.  Without the back-door or another valid identification
argument, the intersection need not be a singleton.  Kan extension packages
the agreement test; consistency, positivity, exchangeability, graphical
criteria, and do-calculus establish the agreement \citep{pearl-book}.

\section{Broader Impact}
\label{sec:broader_impact}

This work is theoretical and does not introduce a deployed system, dataset, or empirical model.  Its positive impact is conceptual: it may help researchers compare decision-making methods across reinforcement learning, causal inference, planning, and game theory using shared semantic criteria.  The main risk is misuse through overinterpretation: a universal categorical semantics does not by itself guarantee correct modeling assumptions, safe optimization, or reliable behavior in deployed decision systems.  Any applied use of UDL-style abstractions should therefore be paired with domain-specific validation, uncertainty analysis, and safeguards appropriate to the decision context.


\begin{thebibliography}{}

\bibitem[Bansal {\em et~al.}(2020)Bansal, Jiang, Singla, and Sinha]{DBLP:conf/stoc/BansalJ0S20}
Bansal, N., Jiang, H., Singla, S., and Sinha, M. (2020).
\newblock Online vector balancing and geometric discrepancy.
\newblock In K.~Makarychev, Y.~Makarychev, M.~Tulsiani, G.~Kamath, and J.~Chuzhoy, editors, {\em Proccedings of the 52nd Annual {ACM} {SIGACT} Symposium on Theory of Computing, {STOC} 2020, Chicago, IL, USA, June 22-26, 2020\/}, pages 1139--1152. {ACM}.

\bibitem[Bertsekas(2026)Bertsekas]{bertsekas:rlbook}
Bertsekas, D. (2026).
\newblock {\em A Course in Reinforcement Learning: 2nd Edition\/}.
\newblock Athena Scientific.

\bibitem[Dayan(1993)Dayan]{dayan1993successor}
Dayan, P. (1993).
\newblock Improving generalization for temporal difference learning: The successor representation.
\newblock {\em Neural Computation\/}, {\bf 5}(4), 613--624.

\bibitem[Jacobs(2016)Jacobs]{jacobs:book}
Jacobs, B. (2016).
\newblock {\em Introduction to Coalgebra: Towards Mathematics of States and Observation\/}, volume~59 of {\em Cambridge Tracts in Theoretical Computer Science\/}.
\newblock Cambridge University Press.

\bibitem[Kan(1958)Kan]{kan}
Kan, D. (1958).
\newblock Adjoint functors.
\newblock {\em Transactions of the American Mathematical Society\/}, {\bf 87}(2), 294--329.

\bibitem[Kozen and Ruozzi(2009)Kozen and Ruozzi]{kozen}
Kozen, D. and Ruozzi, N. (2009).
\newblock Applications of metric coinduction.
\newblock {\em Logical Methods in Computer Science\/}, {\bf 5}(3).

\bibitem[Kushner and Yin(2003)Kushner and Yin]{kushner2003stochastic}
Kushner, H.~J. and Yin, G.~G. (2003).
\newblock {\em Stochastic Approximation and Recursive Algorithms and Applications\/}.
\newblock Stochastic Modelling and Applied Probability. Springer.

\bibitem[Littman {\em et~al.}(2001)Littman, Sutton, and Singh]{littman2001predictive}
Littman, M.~L., Sutton, R.~S., and Singh, S.~P. (2001).
\newblock Predictive representations of state.
\newblock In {\em Advances in Neural Information Processing Systems 14\/}.

\bibitem[Mac~Lane(1971)Mac~Lane]{maclane:71}
Mac~Lane, S. (1971).
\newblock {\em Categories for the Working Mathematician\/}.
\newblock Springer-Verlag, New York.
\newblock Graduate Texts in Mathematics, Vol. 5.

\bibitem[Mahadevan(2005)Mahadevan]{mahadevan2005proto}
Mahadevan, S. (2005).
\newblock Proto-value functions: Developmental reinforcement learning.
\newblock In {\em Proceedings of the 22nd International Conference on Machine Learning\/}, pages 553--560.

\bibitem[Mahadevan(2021)Mahadevan]{sm:udm}
Mahadevan, S. (2021).
\newblock Universal decision models.
\newblock {\em CoRR\/}, {\bf abs/2110.15431}.

\bibitem[Mahadevan(2026)Mahadevan]{mahadevancategoriesforagi}
Mahadevan, S. (2026).
\newblock Categories for {AGI}.

\bibitem[Maschler {\em et~al.}(2013)Maschler, Solan, and Zamir]{Maschler_Solan_Zamir_2013}
Maschler, M., Solan, E., and Zamir, S. (2013).
\newblock {\em Game Theory\/}.
\newblock Cambridge University Press.

\bibitem[Nash(1951)Nash]{nash}
Nash, J. (1951).
\newblock Non-cooperative games.
\newblock {\em Annals of Mathematics\/}, {\bf 54}(2), 286--295.

\bibitem[Pearl(2009)Pearl]{pearl-book}
Pearl, J. (2009).
\newblock {\em Causality: Models, Reasoning and Inference\/}.
\newblock Cambridge University Press, 2nd edition.

\bibitem[Riehl(2017)Riehl]{riehl2017category}
Riehl, E. (2017).
\newblock {\em Category Theory in Context\/}.
\newblock Dover Publications.

\bibitem[Rutten(2000)Rutten]{rutten2000universal}
Rutten, J. J. M.~M. (2000).
\newblock Universal coalgebra: A theory of systems.
\newblock {\em Theoretical Computer Science\/}, {\bf 249}(1), 3--80.

\bibitem[Samuel(1959)Samuel]{samuel1959}
Samuel, A.~L. (1959).
\newblock Some studies in machine learning using the game of checkers.
\newblock {\em IBM Journal of Research and Development\/}, {\bf 3}(3), 210--229.

\bibitem[Sokolova(2011)Sokolova]{SOKOLOVA20115095}
Sokolova, A. (2011).
\newblock Probabilistic systems coalgebraically: A survey.
\newblock {\em Theoretical Computer Science\/}, {\bf 412}(38), 5095--5110.

\bibitem[Sutton and Barto(1998)Sutton and Barto]{DBLP:books/lib/SuttonB98}
Sutton, R.~S. and Barto, A.~G. (1998).
\newblock {\em Reinforcement Learning: An Introduction\/}.
\newblock MIT Press.

\bibitem[Valiant(1984)Valiant]{DBLP:journals/cacm/Valiant84}
Valiant, L.~G. (1984).
\newblock A theory of the learnable.
\newblock {\em Communications of the ACM\/}, {\bf 27}(11), 1134--1142.

\bibitem[von Neumann and Morgenstern(1947)von Neumann and Morgenstern]{vonneumann1947}
von Neumann, J. and Morgenstern, O. (1947).
\newblock {\em Theory of Games and Economic Behavior\/}.
\newblock Princeton University Press.

\bibitem[Witsenhausen(1975)Witsenhausen]{witsenhausen:1975}
Witsenhausen, H.~S. (1975).
\newblock The intrinsic model for discrete stochastic control: Some open problems.
\newblock In {\em Control Theory, Numerical Methods and Computer Systems Modelling\/}, volume 107 of {\em Lecture Notes in Economics and Mathematical Systems\/}, pages 322--335. Springer.

\end{thebibliography}
\end{document}